\pdfoutput=1
\documentclass{article}
\usepackage{spconf}
\usepackage{amsmath,amssymb,dsfont,graphicx,verbatim}
\graphicspath{{figures/}}
\usepackage{amsmath,amssymb,dsfont,verbatim}
\usepackage{import}
\usepackage{caption}
\usepackage{subcaption}
\usepackage[ruled,vlined]{algorithm2e}
\usepackage{cite}
\usepackage{xcolor}
\usepackage{enumitem}
\usepackage{booktabs} 
\usepackage{tikz}
\usepackage{siunitx}
\usepackage{balance}
\usepackage{hyperref}
\usetikzlibrary{mindmap}

\allowdisplaybreaks[0]

\DeclareMathOperator{\cw}{{\scriptstyle\mathcal{W}}}
\DeclareMathOperator{\bcw}{{\boldsymbol{\scriptstyle\mathcal{W}}}}
\usepackage{amsthm}

\DeclareMathOperator{\w}{\boldsymbol{w}}
\DeclareMathOperator{\x}{\boldsymbol{x}}
\DeclareMathOperator{\s}{\boldsymbol{s}}

\usepackage{amsthm}

\theoremstyle{plain}

\newtheorem{assumption}{Assumption}
\newtheorem{theorem}{Theorem}

\newtheorem{lemma}{Lemma}

\let\OLDthebibliography\thebibliography
\renewcommand\thebibliography[1]{
  \OLDthebibliography{#1}
  \setlength{\parskip}{0pt}
  \setlength{\itemsep}{1pt plus 0.6ex}
}

\title{Exact Subspace Diffusion for Decentralized Multitask Learning}
\name{Shreya Wadehra\(^{\star}\), Roula Nassif\(^{\dagger}\) and Stefan Vlaski\(^{\star}\)}
\address{\(^{\star}\) Department of Electrical and Electronic Engineering, Imperial College London, UK\thanks{Emails: shreya.wadehra22@imperial.ac.uk, roula.nassif@unice.fr, s.vlaski@imperial.ac.uk}\\\(^{\dagger}\) I3S Laboratory, Universit\'e C\^ote d'Azur, France}

\begin{document}
\maketitle
\begin{abstract}
  Classical paradigms for distributed learning, such as federated or decentralized gradient descent, employ consensus mechanisms to enforce homogeneity among agents. While these strategies have proven effective in i.i.d. scenarios, they can result in significant performance degradation when agents follow heterogeneous objectives or data. Distributed strategies for multitask learning, on the other hand, induce relationships between agents in a more nuanced manner, and encourage collaboration without enforcing consensus. We develop a generalization of the exact diffusion algorithm for subspace constrained multitask learning over networks, and derive an accurate expression for its mean-squared deviation when utilizing noisy gradient approximations. We verify numerically the accuracy of the predicted performance expressions, as well as the improved performance of the proposed approach over alternatives based on approximate projections.
\end{abstract}
\section{Introduction}\label{sec:intro}
We consider a collection of \( K \) agents, indexed by \( k \), where each agent is equipped with a local risk \( J_k(w_k) \) of the form:
\begin{align}
  J_k(w_k) = \mathds{E} Q_k(w_k; \x_k)
\end{align}
Here, \( w_k \in \mathds{R}^M \) denotes some parameterization chosen by agent \( k \) and \( \x_k \) describes locally available data, modeled as a random variable. The loss \( Q_k(w_k; \x_k) \) then measures how well \( w_k \) fits to the data \( \x_k \), and the risk \( J_k(w_k) \) measures the expected loss.

In a networked setting, where agents have the ability to communicate and exchange information, we may then consider various learning paradigms. As a baseline, in a non-cooperative mode of operation, agents pursue locally optimal models independently by minimizing their local risks:
\begin{align}\label{eq:noncooperative}
  w_k^o \triangleq \arg\min_{w_k \in \mathds{R}^M} J_k(w_k)
\end{align}
An alternative to non-cooperative approaches is the consensus optimization problem~\cite{Tsitsiklis86, Bertsekas97parallel, Nedic09, Chen13, Kairouz21, Sayed22}:
\begin{align}\label{eq:singletask}
  w^o \triangleq \arg\min_{w \in \mathds{R}^M} \frac{1}{K} \sum_{k=1}^K J_k(w)
\end{align}
In contrast to the independent learning objectives in~\eqref{eq:noncooperative}, local objectives are coupled in~\eqref{eq:singletask} into a single task of coming to consensus on an optimal average model \( w^o \). For this reason, the problem of consensus optimization is also referred to as single-task learning. Depending on restrictions on the flow of information, solutions to~\eqref{eq:singletask} may be pursued using a number of different learning strategies, ranging from centralized or parallel~\cite{Bertsekas97parallel} to federated~\cite{Kairouz21} or fully-decentralized structures~\cite{Tsitsiklis86, Nedic09, Chen13, Sayed14, Shi15, Yuan18, DiLorenzo16, Mokhtari16,Koloskova20:AUnifiedTheoryOfDecentralized}.

When agents are homogeneous, meaning that their local data \( \x_k \) are identically and independently distributed, and they employ the same loss \( Q_k(\cdot; \cdot) \), it follows that the models \( w_k^o \) and \( w^o \) coincide. In such settings, \( K \)-fold improvement in performance has been established for a number of different algorithms, and a number of different performance metrics, including the mean-squared deviation for convex risks~\cite{Chen15Performance} or first-order stationarity~\cite{Lian17} and second-order stationarity~\cite{Vlaski19nonconvexP2} for non-convex risks. 

In heterogeneous settings, performance trade-offs become more nuanced, since solving the consensus optimization problem~\eqref{eq:singletask} induces a bias relative to locally optimal models~\eqref{eq:noncooperative}. Depending on the level of heterogeneity in the network, this bias can outweigh any benefit gained from cooperation, and result in a consensus model \( w^o \), which is optimal on average, but performs poorly on any local objective. This observation has motivated the inclusion of tools from multitask learning~\cite{Caruana97} in both federated~\cite{Smith17:FederatedMulti-TaskLearning} and decentralized settings~\cite{Barbarossa09:DistributedSignalSubspace,Koppel17:ProximityWithoutConsensus, Nassif20, Nassif20:LearningOverMultitaskGraphs,Nassif20:AdaptationAndLearningOverNetworksUnderSubspaceConstraintsI, DiLorenzo20:DistributedSignalProcessing,Kayaalp22:Dif-MAML}. Strategies for distributed multitask learning couple local objectives without enforcing strict consensus. While various models have been considered in the literature, we will focus in this work on subspace constrained multitask learning problems of the form~\cite{Barbarossa09:DistributedSignalSubspace,Nassif20:AdaptationAndLearningOverNetworksUnderSubspaceConstraintsI,DiLorenzo20:DistributedSignalProcessing}:
\begin{align}\label{eq:subspace_problem}
  \cw^{\star} \triangleq \arg\min_{\cw \in \mathds{R}^{KM} } \sum_{k=1}^K J_k(w_k) \ \textrm{subject to } \cw \in \mathcal{R}\left( \mathcal{U} \right)
\end{align}
Here, we denote by \( \mathcal{R}\left( \mathcal{U} \right) \) the range space of some matrix \( \mathcal{U} \in \mathds{R}^{KM \times P} \) with full column-rank and introduce the network quantity \( \cw \), which is obtained by stacking the individual models \( w_k \) as in:
\begin{align}
  \cw \triangleq \mathrm{col}\left\{ w_k \right\} 
\end{align}
Note that choosing \( \mathcal{U} = I_{KM} \) results in \( \mathcal{R}\left( \mathcal{U} \right) = \mathds{R}^{KM} \), and hence
\begin{align}
  \cw^{\star} \triangleq \arg\min_{\cw \in \mathds{R}^{KM} } \sum_{k=1}^K J_k(w_k)
\end{align}
decouples into the independent local optimization problems of~\eqref{eq:noncooperative}. In this case, \( \cw^{\star} \) coincides with \( \cw^{o} \triangleq \mathrm{col}\left\{ w_k^o \right\} \), which is obtained by stacking the solutions \( w_k^o \) of~\eqref{eq:noncooperative}. On the other hand, if we let \( \mathcal{U} = \mathds{1}_K \otimes I_M \), it can be verified~\cite{Nassif20:AdaptationAndLearningOverNetworksUnderSubspaceConstraintsI} that:
\begin{align}
  \cw \in \mathcal{R}\left( \mathds{1}_K \otimes I_M \right) \Longleftrightarrow w_{k} = w_{\ell} \ \textrm{for all } k, \ell.
\end{align}
Hence, in this case~\eqref{eq:subspace_problem} reduces to the consensus optimization problem~\eqref{eq:singletask}. Other choices of \( \mathcal{U} \) result in alternative task-relationship priors by restricting models \( w_k \) contained in \( \cw \) to lie in a lower-dimensional linear subspace spanned by \( \mathcal{U} \). For example, setting \( \mathcal{U} \) to the leading eigenvectors of the graph Fourier transform of an underlying graph~\cite{Dong20} results in bandlimited recovery, while other choices of \( \mathcal{U} \) can be used to encode overlap or pairwise linear constraints between agents. We refer the reader to~\cite{DiLorenzo20:DistributedSignalProcessing,Nassif20:AdaptationAndLearningOverNetworksUnderSubspaceConstraintsI} for further examples and details on how to match the choice of subspace constraint to the underlying signal processing, learning or optimization problem.

\section{Problem Formulation and Related Works}
\subsection{Network Model}
Each agent in the network is represented by a node in a graph, where edges between agents represent communication links, meaning that information may be exchanged between this pair of agents. We denote by \( \mathcal{N} \) the set of all agents, and by \( \mathcal{N}_{k} \) the neighborhood of agent \( k \). In other words, \( \ell \in \mathcal{N}_{k} \) implies that there is an edge from agent \( \ell \) to agent \( k \). We will assume the graph to be undirected, meaning that \( \ell \in \mathcal{N}_k \Longleftrightarrow k \in \mathcal{N}_{\ell}\).

\subsection{Approximate Projection-Based Algorithms}
We begin by briefly reviewing the derivations and algorithms of~\cite{Nassif20:AdaptationAndLearningOverNetworksUnderSubspaceConstraintsI,DiLorenzo20:DistributedSignalProcessing}, which are closely related to our proposed algorithm, before pointing out their limitations. Since \( \mathcal{R}(\mathcal{U}) \) denotes a linear subspace, its projection can be determined in closed form, and is expressed as:
\begin{align}
  \mathcal{P}_{\mathcal{U}} = \mathcal{U} {\left( \mathcal{U}^{\mathsf{T}}\mathcal{U} \right)}^{-1} \mathcal{U}^{\mathsf{T}}
\end{align}
It is then, at least in principle, possible to pursue \( \cw^{\star} \) in~\eqref{eq:subspace_problem} via projected gradient descent, which takes the form:
\begin{align}\label{eq:deterministic_centralized}
  \cw_i = \mathcal{P}_{\mathcal{U}} \left( \cw_{i-1} - \mu \nabla \mathcal{J}(\cw_{i-1}) \right)
\end{align}
where we defined:
\begin{align}
  \nabla \mathcal{J}(\cw_{i-1}) \triangleq \mathrm{col}\left\{ \nabla J_k(w_{k, i-1}) \right\}
\end{align}
Two factors limit the applicability of~\eqref{eq:deterministic_centralized} in networked learning environments. First, the projection matrix \( \mathcal{P}_{\mathcal{U}} \) is in general dense, meaning that an iteration of~\eqref{eq:deterministic_centralized} would require central aggregation of the local models \( w_{k, i-1} \) contained in \( \cw_{i-1} \). Second, a step of projected gradient descent requires local access to the exact gradients \( \nabla J_k(w_{k, i-1}) = \nabla \mathds{E} Q(w_{k, i-1}; \x_k) \), which in turn requires knowledge of the distribution of \( \x_k \). This is unavailable when learning from finite or streaming realizations of data. The first limitation is addressed in both~\cite{Nassif20:AdaptationAndLearningOverNetworksUnderSubspaceConstraintsI} and~\cite{DiLorenzo20:DistributedSignalProcessing} by replacing the dense projection matrix by a block-sparse approximation \( \mathcal{A} \), which satisfies:
\begin{align}
  \lim_{i \to \infty} \mathcal{A}^i =&\: \mathcal{P}_{\mathcal{U}} \label{eq:limiting} \\
  A_{\ell k} \triangleq&\: [\mathcal{A}]_{\ell k} = 0,\ \ \ \textrm{if}\ \ell \notin \mathcal{N}_k \label{eq:sparsity}
\end{align}
Here, \( [\mathcal{A}]_{\ell k} \) denotes \( \ell k \)-th block of \( \mathcal{A} \) of dimension \( M \times M \). It has been shown that equivalent conditions on \( \mathcal{A} \) are given by~\cite{Barbarossa09:DistributedSignalSubspace,Nassif20:AdaptationAndLearningOverNetworksUnderSubspaceConstraintsI,DiLorenzo20:DistributedSignalProcessing}:
\begin{align}
  \mathcal{A} \mathcal{P}_{\mathcal{U}} =&\: \mathcal{P}_{\mathcal{U}} \label{eq:prop_ap}\\
  \mathcal{P}_{\mathcal{U}} \mathcal{A} =&\: \mathcal{P}_{\mathcal{U}} \label{eq:prop_pa}\\
  \lambda_{\mathcal{A}} \triangleq&\: \rho\left( \mathcal{P}_{\mathcal{U}} - \mathcal{A} \right) < 1 \label{eq:prop_spectral}
\end{align}
Valid choices of \( \mathcal{A} \) can be constructed for a given \( \mathcal{U} \) and sufficiently connected network topology by solving a convex optimization problem~\cite{DiLorenzo20:DistributedSignalProcessing,Nassif20:AdaptationAndLearningOverNetworksUnderSubspaceConstraintsI}. For simplicity, we will assume throughout that \( \mathcal{A} \) is constructed to be symmetric. If not, we could simply make the replacement \( \mathcal{A} \Leftarrow \frac{1}{2}\left( \mathcal{A} + \mathcal{A}^{\mathsf{T}} \right) \). The algorithm of~\cite{Nassif20:AdaptationAndLearningOverNetworksUnderSubspaceConstraintsI} is obtained by directly replacing \( \mathcal{P}_{\mathcal{U}} \) by \( \mathcal{A} \) in~\eqref{eq:deterministic_centralized} to obtain:
\begin{align}\label{eq:ours_network}
  \cw_i = \mathcal{A} \left( \cw_{i-1} - \mu \nabla \mathcal{J}(\cw_{i-1}) \right)
\end{align}
We may return to node-level quantities by exploiting the block-structures of \( \left\{ \mathcal{A}, \cw_{i-1}, \mathcal{J}(\cw_{i-1}) \right\} \) to obtain:
\begin{align}\label{eq:ours_node}
  w_{k, i} = \sum_{\ell \in \mathcal{N}_{k}} A_{\ell k} \left( w_{\ell, i-1} - \mu \nabla J_{\ell}(w_{\ell, i-1}) \right)
\end{align}
The limiting condition~\eqref{eq:limiting} ensures that for small step-sizes \( \mu \), recursion~\eqref{eq:ours_network} approximates~\eqref{eq:deterministic_centralized}, while the sparsity condition~\eqref{eq:sparsity} on the other hand ensures that~\eqref{eq:ours_network} or~\eqref{eq:ours_node} can be implemented by relying only on communication exchanges between neighboring agents. The DiSPO algorithm of~\cite{DiLorenzo20:DistributedSignalProcessing}, on the other hand, applies \( \mathcal{A} \) only to \( \cw_{i-1} \), resulting in:
\begin{align}\label{eq:paolo_network}
  \cw_i = \mathcal{A} \cw_{i-1} - \mu \nabla \mathcal{J}(\cw_{i-1})
\end{align}
or in local quantities:
\begin{align}\label{eq:paolo_node}
  w_{k, i} = \sum_{\ell \in \mathcal{N}_{k}} A_{\ell k} w_{\ell, i-1} - \mu \nabla J_k(w_{k, i-1})
\end{align}
We remark that~\eqref{eq:ours_node} and~\eqref{eq:paolo_node} can be viewed as generalizations of the Adapt-then-Combine (ATC) diffusion algorithm~\cite{Chen13, Sayed14} and distributed gradient descent~\cite{Tsitsiklis86} respectively, where the typical scalar combination weights are replaced by linear transformations \( A_{\ell k} \). It is precisely these linear transformations that allow the decentralized algorithms~\eqref{eq:ours_node} and~\eqref{eq:paolo_node} to solve generic subspace constrained problems of the form~\eqref{eq:subspace_problem}, rather than consensus problems of the form~\eqref{eq:singletask}.

The second limitation is addressed in~\cite{Nassif20:AdaptationAndLearningOverNetworksUnderSubspaceConstraintsI} by replacing true gradients in~\eqref{eq:ours_network} by stochastic approximations \( \widehat{\nabla J}_k(\w_{k, i-1}) \), based on data available at time \( i \). This results in:
\begin{align}
  \bcw_i = \mathcal{A} \left( \bcw_{i-1} - \mu \widehat{\nabla \mathcal{J}}(\bcw_{i-1}) \right) 
\end{align}
or
\begin{align}\label{eq:ours_node_stochastic}
  \w_{k, i} = \sum_{\ell \in \mathcal{N}_{k}} A_{\ell k} \left( \w_{\ell, i-1} - \mu \widehat{\nabla J}_{\ell}(\w_{\ell, i-1}) \right)
\end{align}
where we now employ bold font for \( \bcw_{i} \) or \( \w_{k, i} \) to emphasize that iterates will be random.

\subsection{Primal-Dual Algorithms}
Both ATC-diffusion and the decentralized gradient descent algorithm for consensus optimization are known to exhibit a fixed-point bias~\cite{Chen15transient, Yuan16:OnTheConvergence}. This has motivated a number of approaches for bias-correction in the context of consensus optimization using arguments based on Lagrangian duality and gradient tracking~\cite{Shi15, Yuan18, DiLorenzo16}. Although bias-correction was originally motivated in deterministic settings, its potential benefit has also been established in stochastic settings~\cite{Yuan20}.

A similar bias has been observed in the context of subspace constrained optimization for~\eqref{eq:paolo_network} in~\cite{DiLorenzo20:DistributedSignalProcessing}, and can be verified for~\eqref{eq:ours_network} using similar arguments. These considerations motivate the development of bias-corrected strategies for decentralized subspace constrained optimization. A bias-corrected version of DiSPO for deterministic optimization, termed EDiSPO, is provided in~\cite{DiLorenzo20:DistributedSignalProcessing} by adjusting the arguments that lead to EXTRA~\cite{Shi15}. The recent work~\cite{Marquis22} provides generalizations of a large number of bias-corrected algorithms for consensus optimization to the subspace constrained setting and provides convergence and sensitivity analysis in the presence of i.i.d. perturbations using the integral quadratic constraint framework. In relation to these related works, we make the following contributions:
\begin{itemize}[leftmargin=*]
  \setlength\itemsep{0em}
  \item We derive an exact subspace diffusion algorithm by extending the incremental arguments of~\cite{Yuan18} to the subspace constrained setting. Incremental arguments of this type have been shown to yield wider stability ranges~\cite{Sayed14, Yuan18}.
  \item We allow for stochastic gradient approximations in lieu of true gradients, which induces gradient noise. In contrast to the perturbations in~\cite{Marquis22}, this type of gradient noise is Markovian, rather than independent and identically distributed.
  \item We allow for multiple local updates to take place in between every communication exchange. This flexibility can improve communication efficiency in federated and decentralized settings~\cite{Kairouz21,Koloskova20:AUnifiedTheoryOfDecentralized}.
  \item When agents perform a single update per exchange, we derive an expression for the limiting mean-squared deviation of the proposed algorithm, which matches the centralized benchmark, and show numerically that it approximates the true performance over a wide range of conditions.
\end{itemize}

\section{Algorithm Development}
The derivation essentially mirrors that of~\cite{Yuan18}, after accounting for the more general subspace constraints, and allowing for multiple primal updates along the aggregate objective \( \sum_{k=1}^K J_k(w_k) \). To this end, note that~\eqref{eq:subspace_problem} is equivalent to:
\begin{align}
  \cw^{\star} \triangleq \arg\min_{\cw \in \mathds{R}^{KM} } \sum_{k=1}^K J_k(w_k) \ \textrm{s.t. } \left( I_{KM} - \mathcal{P}_{\mathcal{U}} \right) \cw = 0
\end{align}
As long as \( \mathcal{A} \) satisfies~\eqref{eq:limiting}--\eqref{eq:sparsity}, this is further equivalent to~\cite{DiLorenzo20:DistributedSignalProcessing}:
\begin{align}
  \cw^{\star} \triangleq \arg\min_{\cw \in \mathds{R}^{KM} } \sum_{k=1}^K J_k(w_k) \ \textrm{s.t. } \left( I_{KM} - \mathcal{A} \right) \cw = 0
\end{align}
Analogously to~\cite{Yuan18}, we introduce the augmented Lagrangian:
\begin{align}
  \mathcal{L}\left( \cw, \lambda \right) \triangleq \sum_{k=1}^K J_k(w_k) + \frac{1}{\mu} \lambda^{\mathsf{T}} \mathcal{B} \cw + \frac{1}{4 \mu} \cw^{\mathsf{T}} \left(I_{KM}-\mathcal{A}\right) \cw
\end{align}
Here, \( \mathcal{B} \) denotes the square root of the matrix \( \frac{1}{2} \left( I_{KM} - \mathcal{A} \right) = \mathcal{B}\cdot\mathcal{B} \), which exists as long as \( \mathcal{A} \) is symmetric with spectral radius bounded by one. We can pursue a saddle-point to the augmented Lagrangian via incremental stochastic gradient descent-ascent, initializing \( \boldsymbol{\psi}_{i, 0} = \bcw_{i-1} \) and letting:
\begin{align}
  \boldsymbol{\psi}_{i, e} =&\: \boldsymbol{\psi}_{i, e-1} - \frac{\mu}{E} \widehat{\nabla \mathcal{J}}(\boldsymbol{\psi}_{i, e-1}) \ \mathrm{for}\ e = 1, \ldots, E \label{eq:primal_dual_1}\\
  \bcw_i =&\: \overline{\mathcal{A}} \boldsymbol{\psi}_{i, E}  - \mathcal{B} \boldsymbol{\lambda}_{i-1} \label{eq:primal_dual_2}\\
  \boldsymbol{\lambda}_i =&\: \boldsymbol{\lambda}_{i-1} + \mathcal{B} \bcw_{i}\label{eq:primal_dual_3}
\end{align}
where we defined \( \overline{\mathcal{A}} = \frac{1}{2}\left( I_{KM} + \mathcal{A} \right) \). Evaluating~\eqref{eq:primal_dual_2} at time \( i-1 \) yields:
\begin{align}\label{eq:intermediate}
  \bcw_{i-1} =&\: \overline{\mathcal{A}} \boldsymbol{\psi}_{i-1, E} - \mathcal{B} \boldsymbol{\lambda}_{i-2}
\end{align}
Subtracting~\eqref{eq:intermediate} from~\eqref{eq:primal_dual_2}, and using~\eqref{eq:primal_dual_3}:
\begin{align}
  &\: \bcw_i - \bcw_{i-1} \notag \\
  =&\: \overline{\mathcal{A}} \boldsymbol{\psi}_{i, E} - \overline{\mathcal{A}} \boldsymbol{\psi}_{i-1, E} - \mathcal{B} \left(  \boldsymbol{\lambda}_{i-1} - \boldsymbol{\lambda}_{i-2} \right) \notag\\
  =&\: \overline{\mathcal{A}} \boldsymbol{\psi}_{i, E} - \overline{\mathcal{A}} \boldsymbol{\psi}_{i-1, E} - \mathcal{B}^2 \bcw_{i-1} \notag\\
  =&\: \overline{\mathcal{A}} \boldsymbol{\psi}_{i, E} - \overline{\mathcal{A}} \boldsymbol{\psi}_{i-1, E} - \frac{1}{2} \bcw_{i-1} + \frac{1}{2} \mathcal{A} \bcw_{i-1}
\end{align}
After rearranging, we have:
\begin{align}
  \bcw_i = \overline{\mathcal{A}} \left( \bcw_{i-1} + \boldsymbol{\psi}_{i, E} - \boldsymbol{\psi}_{i-1, E} \right)
\end{align}
Upon returning to local quantities, we obtain the proposed exact subspace diffusion algorithm with local updates in Algorithm~\ref{alg:proposed}. We remark that setting \( E = 1 \) and \( A_{\ell k} = a_{\ell k} I_M \), we recover the exact diffusion algorithm of~\cite{Yuan18}, which justifies the name.
\begin{algorithm}
  Initialize \( \w_{k, 0} \) arbitrary and \( \boldsymbol{\psi}_{k, 0, E} = \w_{k, 0} \). Set:
  \begin{align}
    \overline{\mathcal{A}} = \frac{1}{2} \left( \mathcal{A} + I_{KM} \right)
  \end{align}
  For \( i \ge 1 \), set \( \boldsymbol{\psi}_{k, i, 0} = \w_{k, i-1} \) and perform \( E \) local updates for \( e = 1, \dots, E \):
  \begin{align}
    \boldsymbol{\psi}_{k, i, e} =&\: \boldsymbol{\psi}_{k, i, e-1} - \frac{\mu}{E} \widehat{\nabla {J}}_k(\boldsymbol{\psi}_{k, i, e-1})\label{eq:proposed_1}
  \end{align}
  Correct:
  \begin{align}
    \boldsymbol{\phi}_{k, i} = \w_{k, i-1} + \boldsymbol{\psi}_{k, i, E} - \boldsymbol{\psi}_{k, i-1, E} 
  \end{align}
  Exchange and transform:
  \begin{align}
    \w_{k, i} = \sum_{\ell \in \mathcal{N}_k} \overline{A}_{\ell k} \boldsymbol{\phi}_{\ell, i}\label{eq:proposed_3}
  \end{align}
  \caption{Exact Subspace Diffusion}\label{alg:proposed}
\end{algorithm}

\section{Convergence Analysis}
For simplicity, in this section, we will restrict ourselves to single local updates \( E = 1 \). We can then write~\eqref{eq:primal_dual_1}--\eqref{eq:primal_dual_3} compactly as:
\begin{align}
  \bcw_i =&\: \overline{\mathcal{A}} \bcw_{i-1} - \mu \overline{\mathcal{A}} \widehat{\nabla \mathcal{J}}(\bcw_{i-1}) - \mu \mathcal{B} \boldsymbol{\lambda}_{i-1} \label{eq:compact_1}\\
  \boldsymbol{\lambda}_i =&\: \boldsymbol{\lambda}_{i-1} + \mathcal{B} \bcw_{i}\label{eq:compact_2}
\end{align}
Throughout this section, we will be interested in quantifying the mean-squared deviation (MSD) of the network iterates \( \bcw_i \) around \( \cw^{\star} \), defined by~\eqref{eq:subspace_problem}, namely \( \mathds{E} {\|\cw^{\star} - \bcw_i\|}^2\). The network iterates \( \bcw_i \) satisfy the decomposition:
\begin{align}
  \bcw_i = \mathcal{P}_{\mathcal{U}} \bcw_i + \left( I_{KM} - \mathcal{P}_{\mathcal{U}} \right) \bcw_i = \bcw_i^{\mathcal{U}} + \bcw_i^{\perp \mathcal{U}}\label{eq:decompositoin}
\end{align}
Here, we defined:
\begin{align}
  \bcw_i^{\mathcal{U}} \triangleq&\: \mathcal{P}_{\mathcal{U}} \bcw_i \\
  \bcw_i^{\perp \mathcal{U}} \triangleq&\: \left( I_{KM} - \mathcal{P}_{\mathcal{U}} \right) \bcw_i
\end{align}
Since for any projection \( \mathcal{P}_{\mathcal{U}} = \mathcal{P}_{\mathcal{U}}^{2} \), it can be readily verified that \( \bcw_i^{\mathcal{U}} \) and \( \bcw_i^{\perp \mathcal{U}} \) are orthogonal. Using this, and the fact that \( \cw^{\star} \in \mathcal{R}(\mathcal{U})\), it follows that:
\begin{align}
  \mathds{E} {\|\cw^{\star} - \bcw_i\|}^2 =&\: \mathds{E} {\|\cw^{\star} - \bcw_i^{\mathcal{U}} - \bcw_i^{\perp \mathcal{U}}\|}^2 \notag \\
  =&\: \mathds{E} {\|\cw^{\star} - \bcw_i^{\mathcal{U}} \|}^2  + \mathds{E} {\| \bcw_i^{\perp \mathcal{U}}\|}^2 
\end{align}
We may then equivalently express the mean-squared deviation \( \mathds{E} {\|\cw^{\star} - \bcw_i\|}^2 \) by instead quantifying the orthogonal contributions \( \mathds{E} {\|\cw^{\star} - \bcw_i^{\mathcal{U}} \|}^2 \) and \(  \mathds{E} {\| \bcw_i^{\perp \mathcal{U}}\|}^2 \), which is a common theme in the study of decentralized learning algorithms~\cite{Chen15transient,Sayed14,Nassif20:AdaptationAndLearningOverNetworksUnderSubspaceConstraintsI,DiLorenzo20:DistributedSignalProcessing}. Applying the projector \( \mathcal{P}_{\mathcal{U}} \) to~\eqref{eq:compact_1} yields:
\begin{align}
  \bcw_i^{\mathcal{U}} =&\: \mathcal{P}_{\mathcal{U}}\overline{\mathcal{A}} \bcw_{i-1} - \mu \mathcal{P}_{\mathcal{U}}\overline{\mathcal{A}} \widehat{\nabla \mathcal{J}}(\bcw_{i-1}) - \mu \mathcal{P}_{\mathcal{U}}\mathcal{B} \boldsymbol{\lambda}_{i-1} \notag\\
  \stackrel{(a)}{=}&\: \mathcal{P}_{\mathcal{U}} \left( \bcw_{i-1} - \mu\widehat{\nabla \mathcal{J}}(\bcw_{i-1}) \right) \notag \\
  =&\: \bcw_{i-1}^{\mathcal{U}} - \mu \mathcal{P}_{\mathcal{U}} \widehat{\nabla \mathcal{J}}(\bcw_{i-1}) \notag\\
  =&\: \mathcal{P}_{\mathcal{U}} \left( \bcw_{i-1}^{\mathcal{U}} - \mu \widehat{\nabla \mathcal{J}}(\bcw_{i-1}) \label{eq:centroid} \right)
\end{align}
where \( (a) \) follows from the spectral structure of \( \mathcal{A} \), induced by the conditions~\eqref{eq:limiting}--\eqref{eq:sparsity}~\cite{Barbarossa09:DistributedSignalSubspace,DiLorenzo20:DistributedSignalProcessing,Nassif20:AdaptationAndLearningOverNetworksUnderSubspaceConstraintsI}, ensuring that:
\begin{align}
  \mathcal{P}_{\mathcal{U}}\overline{\mathcal{A}} =&\: \frac{1}{2} \mathcal{P}_{\mathcal{U}} \left( I_{KM} + {\mathcal{A}} \right) = \mathcal{P}_{\mathcal{U}} \\
  \mathcal{P}_{\mathcal{U}} \mathcal{B} =&\: \mathcal{P}_{\mathcal{U}} {\left( \frac{1}{2} \left(I_{KM} - \mathcal{A} \right)\right)}^{\frac{1}{2}} = 0\label{eq:intermeiate_prop}
\end{align}
In interpreting recursion~\eqref{eq:centroid}, it is useful to consider a centralized benchmark. In the absence of communication constraints, one could pursue a solution to~\eqref{eq:subspace_problem} via projected stochastic gradient descent, yielding the recursion:
\begin{align}
  \bcw_i^{\mathrm{cent}} =&\: \mathcal{P}_{\mathcal{U}} \left( \bcw_{i-1}^{\mathrm{cent}} - \mu  \widehat{\nabla \mathcal{J}}(\bcw_{i-1}^{\mathrm{cent}}) \right)\label{eq:centralized}
\end{align}
Comparing~\eqref{eq:centralized} to~\eqref{eq:centroid}, we observe that the recursions are almost identical, except that the stochastic gradients in~\eqref{eq:centroid} are evaluated at \( \bcw_{i-1} \) instead of \(\bcw_{i-1}^{\mathcal{U}} \). If \( \bcw_{i-1} \approx \bcw_{i-1}^{\mathcal{U}}\), and under suitable smoothness conditions on \( \widehat{\nabla \mathcal{J}}(\cdot) \), it is then reasonable to expect recursion~\eqref{eq:centroid} to track the centralized benchmark~\eqref{eq:centralized}. In light of~\eqref{eq:decompositoin}, the deviation \( \bcw_i - \bcw_i^{\mathcal{U}} \) is given by \( \bcw_i^{\perp \mathcal{U}} \triangleq \left( I_{KM} - \mathcal{P}_{\mathcal{U}} \right) \bcw_i \). Applying \( \left( I_{KM} - \mathcal{P}_{\mathcal{U}} \right) \) to~\eqref{eq:compact_1}, we find:
\begin{align}
  &\: \bcw_i^{\perp\mathcal{U}} \notag \\
  =&\: \left( I_{KM} - \mathcal{P}_{\mathcal{U}} \right)\overline{\mathcal{A}} \bcw_{i-1} - \mu \left( I_{KM} - \mathcal{P}_{\mathcal{U}} \right)\overline{\mathcal{A}} \widehat{\nabla \mathcal{J}}(\bcw_{i-1}) \notag \\
  &\:- \mu \left( I_{KM} - \mathcal{P}_{\mathcal{U}} \right)\mathcal{B} \boldsymbol{\lambda}_{i-1} \notag \\
  \stackrel{\eqref{eq:prop_pa}}{=}&\: \left( \overline{\mathcal{A}} - \mathcal{P}_{\mathcal{U}} \right) \bcw_{i-1} - \mu \left( \overline{\mathcal{A}} - \mathcal{P}_{\mathcal{U}} \right)\widehat{\nabla \mathcal{J}}(\bcw_{i-1}) - \mu \mathcal{B} \boldsymbol{\lambda}_{i-1} \notag \\
  \stackrel{\eqref{eq:prop_ap}}{=}&\: \left( \overline{\mathcal{A}} - \mathcal{P}_{\mathcal{U}} \right)\left( I_{KM} - \mathcal{P}_{\mathcal{U}}\right) \bcw_{i-1} - \mu \left( \overline{\mathcal{A}} - \mathcal{P}_{\mathcal{U}} \right)\widehat{\nabla \mathcal{J}}(\bcw_{i-1}) \notag \\
  &\:- \mu \mathcal{B} \boldsymbol{\lambda}_{i-1} \notag \\
  {=}&\: \left( \overline{\mathcal{A}} - \mathcal{P}_{\mathcal{U}} \right)\bcw_{i-1}^{\perp\mathcal{U}} - \mu \left( \overline{\mathcal{A}} - \mathcal{P}_{\mathcal{U}} \right)\widehat{\nabla \mathcal{J}}(\bcw_{i-1}) - \mu \mathcal{B} \boldsymbol{\lambda}_{i-1}\label{eq:deviation}
\end{align}
In light of~\eqref{eq:prop_spectral}, we have:
\begin{align}
  \lambda_{\overline{\mathcal{A}}} \triangleq&\: \rho \left( \overline{\mathcal{A}} - \mathcal{P}_{\mathcal{U}} \right) = \rho \left( \frac{1}{2} \left( I_{KM} + {\mathcal{A}} \right) - \mathcal{P}_{\mathcal{U}} \right) \notag \\
  \le&\: \frac{1}{2} \rho \left( I_{KM} - \mathcal{P}_{\mathcal{U}} \right) + \frac{1}{2} \rho \left( \mathcal{A} - \mathcal{P}_{\mathcal{U}} \right) \notag \\
  \le&\: \frac{1}{2} + \frac{1}{2} \lambda_{\mathcal{A}} < 1
\end{align}
It follows that the recursion~\eqref{eq:deviation} in the deviation \( \bcw_i^{\perp\mathcal{U}} \) is driven by a stable matrix. Some care needs to be taken, since the recursions~\eqref{eq:centroid} and~\eqref{eq:deviation} are coupled through the driving term \( \widehat{\nabla \mathcal{J}}(\bcw_{i-1}) \), and~\eqref{eq:deviation} is additionally driven by the dual variable \( \boldsymbol{\lambda}_{i-1}\).

\subsection{Stability}
We introduce the following common modeling conditions on the objectives \( J_k(\cdot) \) and gradient approximations \( \widehat{\nabla J}_k(\cdot) \)~\cite{Sayed14,Yuan20, Nassif20:AdaptationAndLearningOverNetworksUnderSubspaceConstraintsI}.
\begin{assumption}[\textbf{Conditions on \( J_k(\cdot) \)}] Each local objective \( J_k(\cdot) \) is \( \nu_k \)-strongly convex with \( \delta_k \)-Lipschitz gradients, meaning that for every \( x, y \in \mathds{R}^M\), it holds that:
  \begin{align}
    {\left(\nabla J_k(x) - \nabla J_k(y)\right)}^{\mathsf{T}}(x-y) \ge&\: \nu_k \|x - y\|^2\\
    \|\nabla J_k(x) - \nabla J_k(y)\| \le&\: \delta_k \|x - y\|
  \end{align}
  for some \( 0 < \nu_k < \delta_k \).\qed\label{as:regularity}
\end{assumption}
\begin{assumption}[\textbf{Conditions on \(\widehat{\nabla J}_k(\cdot) \)}] Define the local gradient noise process:
  \begin{align}
    \s_{k, i}(\w_{k, i-1}) \triangleq \widehat{\nabla J}_k(\w_{k, i-1}) - {\nabla J}_k(\w_{k, i-1})\label{eq:gradient_noise_process}
  \end{align}
The gradient noise is zero-mean after conditioning on the current iterate:
\begin{align}
  \mathds{E} \left\{ \s_{k, i}(\w_{k, i-1}) | \w_{k, i-1} \right\} = \s_{k, i}(\w_{k, i-1})
\end{align}
Furthermore, its variance is bounded according to:
\begin{align}
  \mathds{E} \left\{ \|\s_{k, i}(\w_{k, i-1})\|^2 | \w_{k, i-1} \right\} \le \beta_k^2 \|w_k^{\star} - \w_{k, i-1}\|^2 + \sigma_k^2
\end{align}
Finally, the gradient noise processes between agents \( k \neq \ell \) are uncorrelated after conditioning on current iterates:
\begin{align}
  \mathds{E} \left\{ \s_{k, i}(\w_{k, i-1})\s_{\ell, i}(\w_{\ell, i-1})^{\mathsf{T}} | \w_{k, i-1}, \w_{\ell, i-1} \right\} = 0
\end{align}\label{as:gradient_noise}\qed
\end{assumption}
\noindent A common gradient approximation is given by \( \widehat{\nabla J}_k(\w_{k, i-1}) = \nabla Q(\w_{k, i-1}; \x_{k, i}) \), where \( \x_{k, i} \) denotes the sample available to agent \( k \) at time \( i \). Alternative constructions, such as mini-batch or asynchronous updates are possible as well. In those cases, and for many loss functions \( Q(\cdot; \cdot) \) arising in learning problems, the conditions in Assumption~\ref{as:gradient_noise} can be verified to hold --- we refer the reader to~\cite{Sayed22, Vlaski22:NetworkedSignalAnd} for details. We can then establish the following stability result.
\begin{lemma}[\textbf{Mean-square stability}] Under Assumptions~\ref{as:regularity}--\ref{as:gradient_noise}, and for symmetric \( \mathcal{A} \) satisfying~\eqref{eq:prop_ap}--\eqref{eq:prop_spectral}, there exists a step-size \( \mu \) small enough, so that the exact subspace diffusion recursions~\eqref{eq:proposed_1}--\eqref{eq:proposed_3} are stable in the mean-square sense, and:
  \begin{align}\label{eq:stability}
    \limsup_{i \to \infty} \mathds{E} \left\| \cw^{\star} - \bcw_i \right\|^2 = O(\mu)
  \end{align}
\end{lemma}
\begin{proof}
  The proof continues from recursions~\eqref{eq:centroid} and~\eqref{eq:deviation} using arguments similar to~\cite{Yuan18, Yuan20} and is omitted due to space limitations. We verify the claim numerically in Section~\ref{sec:numerical}.
\end{proof}
\noindent The \( O(\cdot) \) notations in~\eqref{eq:stability} denotes that the dependence of \( \limsup_{i \to \infty} \mathds{E} \left\| \cw^{\star} - \bcw_i \right\|^2 \) is approximately linear in \( \mu \) for small step-sizes, or more precisely that
\begin{align}
  \lim_{\mu \to 0} \frac{1}{\mu} \limsup_{i \to \infty} \mathds{E} \left\| \cw^{\star} - \bcw_i \right\|^2
\end{align}
tends towards a finite constant as \( \mu \) tends to zero. To develop a more clear understanding of the performance of the proposed algorithm, we derive an expression for this constant in the next section.

\subsection{Performance}
To provide a more granular understanding of the steady-state performance of the proposed algorithm, we will need to introduce additional regularity conditions on the objectives and gradient approximations. Analogous conditions have been used before when deriving performance expressions of single-task learning algorithms~\cite{Sayed14}, multitask learning algorithms~\cite{Nassif20:LearningOverMultitaskGraphs,Nassif20:AdaptationAndLearningOverNetworksUnderSubspaceConstraintsII}, or when establishing the ability of stochastic gradient algorithms to escape from saddle-points~\cite{Vlaski22:Second-OrderGuarantees}.
\begin{assumption}[\textbf{Higher-order smoothness}] The local objectives \( J_k(\cdot) \) are twice-differentiable, with Lipschitz continuous Hessian matrix \( \nabla^2 J_k(\cdot) \) around \( w_{k}^{\star} \) satisfying for all \( x\):
  \begin{align}
    \|\nabla^2 J_k(x) - \nabla^2 J_k(w_k^{\star}) \| \le \kappa_{k} \|x-w_{k}^{\star}\|
  \end{align}
  We define as \( H_k^{\star} \) the Hessian at the optimal solution of~\eqref{eq:subspace_problem}, namely \( H_k^{\star} \triangleq \nabla J_k(w_k^{\star}) \), and \( \mathcal{H}^{\star} \triangleq \mathrm{diag}\left\{ H_k^{\star} \right\} \). As we will see, \( \mathcal{H}^{\star} \) will play a role in the final performance expression.\qed\label{as:hessian}
\end{assumption}
\begin{assumption}[\textbf{Higher-order gradient noise conditions}] The fourth moment of the gradient noise process~\eqref{eq:gradient_noise_process} also satisfies a relative bound of the form:
  \begin{align}
    \mathds{E} \left\{ \|\s_{k, i}(\w_{k, i-1})\|^4 | \w_{k, i-1} \right\} \le \beta_{4, k}^4 \|w_k^{\star} - \w_{k, i-1}\|^4 + \sigma_{4, k}^4
  \end{align}
  We introduce the gradient noise covariance:
  \begin{align}
    R_{s, k}(\w_{k, i-1}) \triangleq \mathds{E} \left\{ \s_{k, i}(\w_{k, i-1})\s_{k, i}(\w_{k, i-1})^{\mathsf{T}} | \w_{k, i-1} \right\}
  \end{align}
  The gradient noise covariance is also assumed to satisfy a smoothness condition around \( w_{k}^{\star} \):
  \begin{align}
    \|R_{s, k}(x) - R_{s, k}(w_{k}^{\star})\| \le \kappa \|x - w_{k}^{\star}\|^{\gamma}
  \end{align}
  where \( \kappa_k \ge 0 \) and \( 0 < \gamma \le 4 \). We again define \( R_{s, k}^{\star} \triangleq R_{s, k} (w_{k}^{\star}) \) and \( \mathcal{R}_s^{\star} \triangleq \mathrm{diag}\left\{ R_{s, k}^{\star} \right\} \).\qed\label{as:covariance}
\end{assumption}
Given these assumptions, we can then provide a precise characterization of the limiting behavior of the algorithm for small step-sizes.
\begin{theorem}[Mean-squared deviation of exact subspace diffusion] Under Assumptions~\ref{as:regularity}--\ref{as:covariance}, the mean-squared deviation satisfies:
\begin{align}
  \limsup_{i \to \infty} \mathds{E} \|\cw^{\star} - \bcw_i\|^2 = \frac{\mu^2}{K} \mathrm{Tr}\left( \sum_{n=0}^{\infty} \mathcal{C}^n \mathcal{Y} \mathcal{C}^{n} \right) + o(\mu)\label{eq:exact}
\end{align}
where \( o(\mu) \) denotes a higher-order term in \( \mu \) and:
\begin{align}
  \mathcal{C} =&\: \mathcal{A}\left( I_{KM} - \mu \mathcal{H}^{\star} \right) \\
  \mathcal{Y} =&\: \mathcal{A} \mathcal{R}_s^{\star} \mathcal{A}
\end{align}
Asymptotically, for small \( \mu \), this may be approximated by:
\begin{align}
  &\:\mu \lim_{\mu \to 0} \frac{1}{\mu} \limsup_{i \to \infty} \mathds{E} \left\| \cw^{\star} - \bcw_i \right\|^2 \notag \\
  =&\: \frac{\mu}{2K} \mathrm{Tr}\left( {\left( \mathcal{U}^{\mathsf{T}} \mathcal{H}^{\star} \mathcal{U} \right)}^{-1} \mathcal{U}^{\mathsf{T}} \mathcal{R}_s^{\star} \mathcal{U} \right)\label{eq:performance}
\end{align}\label{th:performance}
\end{theorem}
\begin{proof}
  The proof revolves around introducing a long-term model, analogous to~\cite{Sayed14,Nassif20:AdaptationAndLearningOverNetworksUnderSubspaceConstraintsII}, which can be shown to be accurate for small step-sizes and under conditions~\ref{as:hessian}--\ref{as:covariance}. Details are omitted due to space limitations. We verify the accuracy of numerically in Section~\ref{sec:numerical}.
\end{proof}
\noindent We note that for moderately large step-sizes \( \mu \), expression~\eqref{eq:exact} will yield more accurate estimates of the steady-state performance than~\eqref{eq:performance}. Relation~\eqref{eq:performance} on the other hand is more tractable. We can interpret \(  \mathcal{U}^{\mathsf{T}} \mathcal{H}^{\star} \mathcal{U} \) in~\eqref{eq:performance} as the projection of the Hessian onto the space spanned by \( \mathcal{U} \), while \( \mathcal{U}^{\mathsf{T}} \mathcal{R}_s^{\star} \mathcal{U} \) is the projection of the noise covariance onto the same space. In this sense, \( \mathrm{Tr}\left( {\left( \mathcal{U}^{\mathsf{T}} \mathcal{H}^{\star} \mathcal{U} \right)}^{-1} \mathcal{U}^{\mathsf{T}} \mathcal{R}_s^{\star} \mathcal{U} \right) \) is a measure of the inverse signal-to-noise ratio after restricting the signal to the space of feasible solutions in~\eqref{eq:subspace_problem}. It coincides with the performance of centralized benchmark~\eqref{eq:centralized} and the asymptotic performance of~\cite{Nassif20:AdaptationAndLearningOverNetworksUnderSubspaceConstraintsI}. We will verify in Section~\ref{sec:numerical} that~\eqref{eq:exact} is accurate for finite step-sizes \( \mu \), and that the proposed bias-corrected algorithm outperforms the approximate solution of~\cite{Nassif20:AdaptationAndLearningOverNetworksUnderSubspaceConstraintsI}.

\section{Numerical Results}\label{sec:numerical}
We consider the same setting as~\cite[Section~IV]{Nassif20:AdaptationAndLearningOverNetworksUnderSubspaceConstraintsII}, and refer the reader there for a more detailed motivation of the construction. A total of \( K = 50 \) agents are placed uniformly in a \( [0, 1] \times [0, 1] \) square, and weights \( c_{\ell k} \) are assigned between pairs of agents \( \ell \) and \( k \) based on their distance (refer to~\cite[Eq.~(83)]{Nassif20:AdaptationAndLearningOverNetworksUnderSubspaceConstraintsII} for details). Based on these weights, we define the Laplacian matrix \( L = \mathrm{diag}\left\{ C \mathds{1}_{M} \right\} - C \), which in turn defines a graph Fourier transform \( L = V \Lambda V^{\mathsf{T}}\). The local models \( w_k^o \in \mathds{R}^5 \), which make up \( \cw^o \in \mathds{R}^{250} \), are generated by smoothing a randomly generated signal \( \cw \in \mathds{R}^{250} \) through a diffusion kernel defined by the graph Fourier transform of \( L \). This results in a collection of local models \( w_k^o \), which vary smoothly over the graph defined by \( C \). Each agent collects linear observations:
\begin{align}
  \boldsymbol{\gamma}_{k, i} = \boldsymbol{h}_{k, i}^{\mathsf{T}} w_k^{o} + \boldsymbol{v}_{k, i}
\end{align}
where \( \mathds{E} \boldsymbol{h}_{k, i} \boldsymbol{h}_{k, i}^{\mathsf{T}} = \sigma_{h, k}^2 I_5 \), \( \mathds{E} \boldsymbol{v}_{k, i}^2 = \sigma_{v, k}^2 \), and random variables at different agents are independent. The variances \( \sigma_{h, k}^2 \) are sampled from the uniform distribution between \( 0.5 \) and \( 2 \). The local objectives are
\begin{align}
  J_k(w_k) \triangleq \frac{1}{2} \mathds{E}\left( \boldsymbol{\gamma}_{k, i} - \boldsymbol{h}_{k, i}^{\mathsf{T}} w_k \right)^{2}
\end{align}
and we set \( V \) to contain the \( P=3 \) leading eigenvectors of \( U \).

We first illustrate the benefit of bias-correction. To this end, we compare the performance of the approximate projection-based algorithm~\eqref{eq:ours_node_stochastic} of~\cite{Nassif20:AdaptationAndLearningOverNetworksUnderSubspaceConstraintsI} with the proposed exact subspace diffusion algorithms~\eqref{eq:proposed_1}--\eqref{eq:proposed_3} for varying noise profiles. In Fig.~\ref{fig:high_noise}, we sample \( \sigma_{v, k}^2 \) from the uniform distribution between \( 0.2 \) and \( 0.8 \) and observe that both \eqref{eq:ours_node_stochastic} and~\eqref{eq:proposed_1}--\eqref{eq:proposed_3} exhibit similar performance, and match the prediction~\eqref{eq:exact}.
\begin{figure}
  \centering
  \includegraphics[width=\linewidth]{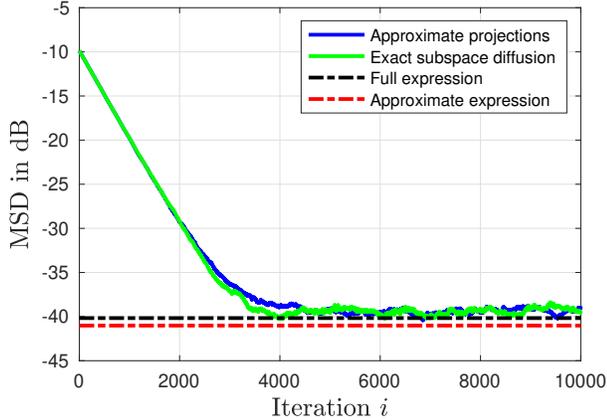}
  \caption{Comparison in the high-noise setting of the approximate projections-based algorithm~\eqref{eq:ours_node_stochastic} and the proposed exact subspace diffusion algorithm~\eqref{eq:proposed_1}--\eqref{eq:proposed_3} with \( E = 1 \). We also show the theoretical predictions based on the full expression~\eqref{eq:exact} and approximation~\eqref{eq:performance}. In a setting with high noise variance \( \sigma_{v, k}^2 \), the performance difference is negligible, since the noise dominates any bias induced by~\eqref{eq:ours_node_stochastic}. Both theoretical expressions match the performance of both algorithms well, with the full expression~\eqref{eq:exact} being more accurate.}\label{fig:high_noise}
\end{figure}

We contrast Fig.~\ref{fig:high_noise} with a second simulation, where the noise power \( \sigma_{v, k}^{2} \) is now sampled from a the uniform distribution between \( 0.2 \cdot 10^{-4} \) and \( 0.8 \cdot 10^{-4} \), and hence much smaller. In Fig.~\ref{fig:low_noise} we observe a significant performance advantage in the proposed approach, while the algorithm~\eqref{eq:ours_node_stochastic} based on approximate projections shows only mild improvement. This bottleneck is due to the bias induced by employing approximate projections, which remains even as the effect of the gradient noise is reduced, becoming increasingly pronounced the noise vanishes. This is consistent with observations made in the context of consensus optimization~\cite{Yuan20}, where bias-correction yields the significant benefit when employing accurate gradient approximations. In all cases, the performance predictions of Theorem~\ref{th:performance} are accurate.
\begin{figure}
  \centering
  \includegraphics[width=\linewidth]{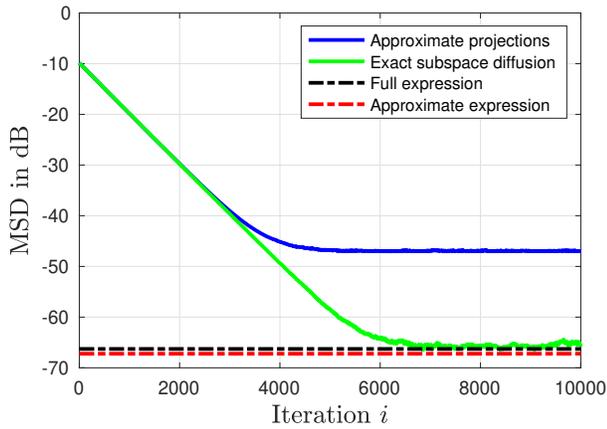}
  \caption{Comparison in the low-noise setting of the approximate projections-based algorithm~\eqref{eq:ours_node_stochastic} and the proposed exact subspace diffusion algorithm~\eqref{eq:proposed_1}--\eqref{eq:proposed_3} with \( E = 1 \). We also show the theoretical predictions based on the full expression~\eqref{eq:exact} and approximation~\eqref{eq:performance}. In a setting with low noise variance \( \sigma_{v, k}^2 \), the performance difference is significant, since the bias induced by~\eqref{eq:ours_node_stochastic} dominates the noise. Both theoretical expressions match the performance of the proposed algorithm well, with the full expression~\eqref{eq:exact} being more accurate.}\label{fig:low_noise}
\end{figure}

Finally, we illustrate the benefit of allowing for multiple local updates in Fig.~\ref{fig:multiple}.
\begin{figure}
  \centering
  \includegraphics[width=\linewidth]{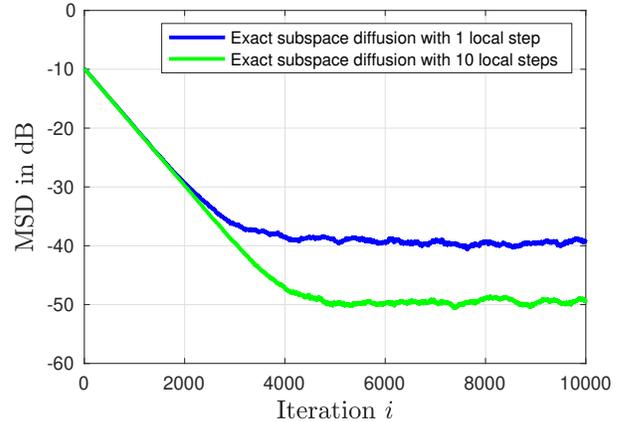}
  \caption{Comparison of the performance of the exact subspace diffusion algorithm~\eqref{eq:proposed_1}--\eqref{eq:proposed_3} with \( E = 1 \) and \( E = 10\). Employing \( 10 \) local updates results in a \( 10 \si{dB}\) gain in performance, corresponding to \( 10 \)-fold reduction in steady-state error.}\label{fig:multiple}
\end{figure}

\section{Conclusion}
We have derived exact subspace diffusion, an algorithm for decentralized subspace constrained multitask learning over networks. The construction removes the bias of approximate schemes using primal-dual arguments, and allows for multiple local updates to allow for improved performance without increasing communication load. The algorithm is complemented with a precise characterization of its steady-state error, and simulations illustrating its advantages as well as the accuracy of performance expressions.

\bibliographystyle{IEEEbib}
{\bibliography{main}}

\end{document}